\documentclass{article}

\usepackage{microtype}
\usepackage{graphicx}
\usepackage{subfigure}
\usepackage{booktabs} 

\usepackage{hyperref}

\usepackage[accepted]{icml2023}

\usepackage{amsmath}
\usepackage{amssymb}
\usepackage{mathtools}
\usepackage{amsthm}

\usepackage[capitalize,noabbrev]{cleveref}

\usepackage{dsfont}

\theoremstyle{plain}
\newtheorem{theorem}{Theorem}[section]
\newtheorem{proposition}[theorem]{Proposition}
\newtheorem{lemma}[theorem]{Lemma}

\theoremstyle{definition}

\theoremstyle{remark}

\usepackage[textsize=tiny]{todonotes}

\newcommand{\argmin}{\text{argmin}}
\newcommand{\argmax}{\text{argmax}}
\newcommand\numberthis{\addtocounter{equation}{1}\tag{\theequation}}

\icmltitlerunning{An Adaptive Sampling Algorithm for Data Market}

\begin{document}

\twocolumn[
\icmltitle{Addressing Budget Allocation and Revenue Allocation in Data Market Environments Using an Adaptive Sampling Algorithm}




\begin{icmlauthorlist}
\icmlauthor{Boxin Zhao}{booth}
\icmlauthor{Boxiang Lyu}{booth}
\icmlauthor{Raul Castro Fernandez}{UchiCS}
\icmlauthor{Mladen Kolar}{booth}
\end{icmlauthorlist}

\icmlaffiliation{booth}{Booth School of Business, University of Chicago, Chicago, IL, USA}
\icmlaffiliation{UchiCS}{Department of Computer Science, University of Chicago, Chicago, IL, USA}

\icmlcorrespondingauthor{Boxin Zhao}{boxinz@uchicago.edu}

\icmlkeywords{Machine Learning, Data Market, Data Valuation, ICML}

\vskip 0.3in
]



\printAffiliationsAndNotice{}  

\begin{abstract}

High-quality machine learning models are dependent on access to high-quality training data. When the data are not already available, it is tedious and costly to obtain them.~\emph{Data markets} help with identifying valuable training data:~model consumers pay to train a model, the market uses that budget to identify data and train the model (the \emph{budget allocation} problem), and finally the market compensates data providers according to their data contribution (\emph{revenue allocation} problem). For example, a bank could pay the data market to access data from other financial institutions to train a fraud detection model. Compensating data contributors requires understanding data's contribution to the model; recent efforts to solve this revenue allocation problem based on the Shapley value are inefficient to lead to practical data markets.

In this paper, we introduce a new algorithm to solve budget allocation and revenue allocation problems simultaneously in linear time. The new algorithm employs an adaptive sampling process that selects data from those providers who are contributing the most to the model. Better data means that the algorithm accesses those providers more often, and more frequent accesses corresponds to higher compensation. Furthermore, the algorithm can be deployed in both centralized and federated scenarios, boosting its applicability. We provide theoretical guarantees for the algorithm that show the budget is used efficiently and the properties of revenue allocation are similar to Shapley's. Finally, we conduct an empirical evaluation to show the performance of the algorithm in practical scenarios and when compared to other baselines. Overall, we believe that the new algorithm paves the way for the implementation of practical data markets.
\end{abstract}

\section{Introduction}

The success of modern machine learning is largely dependent on access to high-quality data.
However, in many scenarios, \emph{consumers} of machine learning (ML) models do not have the necessary data to train them, and collecting data can be impractical; when it is practical it is often tedious and time-consuming. At the same time, data-rich organizations (\emph{providers}) may not be aware that their data would be valuable to prospective consumers.
For example, banks could improve fraud detection if they had access to relevant data from other banks with similar financial products. 
Data markets are a mechanism to address this problem.
A well-functioning \emph{data market} compensates \emph{providers} for the data they supply, and allocates the data to those \emph{consumers}~\citep{fernandez2020data} who need it, thus unleashing the value of data for everyone involved in the transaction. Designing such data markets remains an open problem.

Consider a group of 10 regional banks. One of them (the \emph{consumer}) has detected unusual activity in transactions pertaining to a new credit card product. In response, ML engineers design a fraud detection model for the new product, but realize that they lack training data. These training data could be provided by the other 9 banks, the \emph{providers}, in exchange for some compensation. Clearly, the consumer would like to invest their \emph{budget} in high-quality data that will contribute to better models (the \emph{budget allocation} problem). And providers would like to be compensated according to the value their data brings to the model (the \emph{revenue allocation} problem); e.g., if a bank's data are more relevant to the task at hand, such bank should be compensated higher than the others. In this example, a well-functioning data market aligns the incentives of \emph{consumers} with those of \emph{providers}: it ensures that the budget is used towards obtaining high-quality data.

While there is growing literature on data markets~\citep{fernandez2020data,liu2021dealer,agarwal2019marketplace,li2021data}, no existing data market solves the \emph{budget allocation} and \emph{revenue allocation} problems simultaneously and efficiently. For example, \citet{li2021data} deals with \emph{budget allocation} but does not discuss \emph{revenue allocation}. And there is a plethora of revenue allocation algorithms based on the Shapley value~\citep{shapely1953value} and approximations \citep{agarwal2019marketplace,liu2021dealer,ghorbani2019data,jia2019towards} that, although provide good guarantees on the \emph{fairness} of the allocation, are too computationally expensive to lead to practical solutions.

In this paper, we introduce a new algorithm that solves simultaneously the budget and revenue allocation problems in linear time: much more efficiently than existing revenue allocation approaches. At each training iteration, the algorithm chooses a provider from whom to obtain data. The number of iterations is determined by the budget. Every time the algorithm accesses a provider's data, it compensates that provider using a fixed unit of consumer-provided budget. The key idea behind the algorithm is an \emph{\textbf{O}nline \textbf{S}tochastic \textbf{M}irror \textbf{D}escent (OSMD) sampler }: an adaptive sampling procedure that ensures that the budget is spent on the providers that supply the best data. This means that providers with better data are accessed more often, and thus compensated higher. Most importantly, the new algorithm trains the model only once and has a complexity of $O(n \log n)$, with $n$ being the number of providers. Contrast this with the Shapley value, which must consider the marginal contribution of one provider to all other provider subsets~\citep{deng1994complexity}, leading to a complexity of $\mathcal{O}(2^n)$, and requiring a model training each time.

The key technical intuition behind the new algorithm is that it uses information contained in the optimization/training process to do both revenue allocation and budget allocation.
While the Shapley value (and approximations) treats the optimization process as a black-box and only relies on the output model to do the revenue allocation, the new algorithm instead uses model updates to measure data quality after each iteration. 
This serves as the basis for the new adaptive sampling strategy that ensures that providers with a higher quality history will be accessed more often. By encouraging highly valuable providers to be accessed more, we obtain a better performing model, which then results in better budget allocation.
Crucially, because the new algorithm does not require repeated training, it can be applied to large ML models that are costly to train, which is more practical than previous methods that require multiple re-trainings, and thus are often tested only on simple models, limiting their applicability.

The new algorithm is widely applicable. First, it supports any ML models that are trained in an iterative way. Second, it can be used in centralized and federated scenarios~\citep{kairouz2019advance,mcmahan2017communication}. In centralized scenarios, a central data market platform stores data from providers and trains the model on behalf of consumers. In federated scenarios, data never leave the providers' premise, so it may be more adequate when data are sensitive or subject to regulations. The model is mainly trained by the providers, and the market platform is in charge of running the adaptive sampling algorithm and budget allocation. While in the first scenario, the cost of training the model is absorbed by the data market platform, in the second, it is distributed across providers. We do not deal in this paper with data market architectures or with mechanisms for determining how to use budget to compensate costs. Instead, we note that in both cases, reducing the cost of budget and revenue allocation strictly benefits all scenarios, and we focus on proposing an efficient algorithm.

Finally, we provide theoretical guarantees on both revenue allocation and budget allocation. 
For revenue allocation, we show that the new algorithm has properties similar to the four axioms of the Shapley value~\citep{ghorbani2019data,jia2019efficient}, i.e., efficiency, linearity, null player, and symmetry.
For budget allocation, we show that the average regret of our method compared to spending all budget on the most valuable providers is asymptotically zero as the total budget goes to infinity.
We complement the theoretical guarantees that make the algorithm principled with an extensive empirical evaluation that shows the practical applicability of the new method.

\textbf{Contributions:} We propose an algorithm that achieves revenue allocation and budget allocation at the same time. The proposed algorithm is highly-scalable and achieves similar revenue allocation quality as Shapley-based methods but more efficiently ($O(n \log n)$). Additionally, the new algorithm can be applied to federated learning settings and centralized scenarios seamlessly, increasing its applicability. From a technical perspective, the key technique of the algorithm is an adaptive sampling strategy that uses the information contained in the training process, which is generally ignored by previous research. We provide both theoretical and experimental justifications for our method.

To the best of our knowledge, this is the first highly-scalable method that unifies revenue allocation and budget allocation, thus paving the way to practical data markets.

\textbf{Notation:} For two positive sequences $(a_n)_{n \geq 1}$ and $(b_n)_{n \geq 1}$, we use $a_n=O(b_n)$ to denote that there exists a positive constant $C>0$ such that $a_n \leq C b_n$ when $n$ is large enough.
We use $a_n=o(b_n)$ when $\lim_{n \to \infty} (a_n / b_n) = 0$, and $a_n=o(1)$ when $\lim_{n \to \infty} a_n = 0$.
For $n \geq 1$, we define $\mathcal{P}_{n-1} \coloneqq \left\{ x \in \mathbb{R}^n : \sum^n_{i=1} x_i = 1, x_i \geq 0 \text{ for all } i=1,\ldots,n \right\}$ as the $(n-1)$-dimensional simplex.

\textbf{Organization:} In Section~\ref{sec:related-work}, we summarize the related work. In Section~\ref{sec:prelim}, we introduce our assumptions on the market. We then introduce our proposed method in Section~\ref{sec:OSMD}, and discuss its theoretical guarantees on budget allocation and revenue allocation in Section~\ref{sec:thm-budget} and Section~\ref{sec:thm-rev-alloc}, respectively. The experimental results are in Section~\ref{sec:exp}. We finally conclude our paper with Section~\ref{sec:conclusion}.
The code of the paper is available at:
\url{https://github.com/boxinz17/Data-Market-via-Adaptive-Sampling} .

\section{Related Work}
\label{sec:related-work}

There have been many studies on building a data market from various communities. 
Despite the vast literature and the growing need, there is no practical and principled data market.
We summarize the literature most relevant to our work without attempting to be comprehensive. 
\citet{fernandez2020data} shares the authors' vision of the data market. 
Our project can be regarded as a special case of this vision, where we focus on the data market designed for machine learning applications. 
In addition, \citet{agarwal2019marketplace,li2021data,liu2021dealer} discuss the market design for machine learning, but focus on different aspects from our paper.
\citet{chen2018optimal,kong2020information,zhang2021optimal} offer
complementary work on mechanism design.

Data pricing is a related topic with techniques that are complementary to our paper. \citet{balazinska2011data, lin2014arbitrage, chawla2019revenue, aly2019buy, chen2019towards} discuss pricing mechanisms for data, queries, and models in various problems. See \citet{pei2020survey,fricker2017pricing,liang2018survey} for recent surveys. One may use the pricing methods in the above literature to decide the price of each query of model updates from a data provider or to decide the total number of queries. 

Data valuation is used for revenue allocation and, therefore, is closely related to our work~\citep{koh2017understanding,sharchilev2018finding}. A recent line of research reinterprets the data valuation problem as a cooperative game and applies the Shapley value~\citep{shapely1953value} to quantify the value of data~\citep{ghorbani2019data,jia2019towards,jia2019efficient,ghorbani2020distributional,kwon2021efficient,kwon2022beta}.
In addition, \citet{yoon2020data} proposes a data valuation approach using reinforcement learning (RL). 
The major problem with the above approaches is that they are computationally costly, which restricts their application to simple ML models. Furthermore, the RL-based method needs access to the raw data, and thus is not applicable to privacy-sensitive scenarios such as federated learning.

Adaptive sampling via online learning approach has been applied in optimization literatures~\citep{namkoong2017adaptive,borsos2018online,borsos2019online,el2020adaptive,zhao2021adaptive,zhao2023lsvrg}, where people use it to construct unbiased estimators of the full gradient with an online learning objective to minimize the cumulative variance. In this paper, we adopt a similar idea but the objective is to maximize the cumulative utility gain. While some techniques are similar, this paper solves a fundamentally different problem. 

Finally, several companies have attempted to create practical data markets, such as Scale AI~\citep{scale.ai}, Dawex~\citep{dawex} and Xignite~\citep{xignite}. Our method provides a principled algorithmic approach for practical data markets.

\section{Preliminaries and Market Model}
\label{sec:prelim}

We start by formalizing the three parties involved in the data market: data consumers, data providers, and the market arbiter. We describe their roles and state assumptions on their behavior. Throughout the paper, we focus on serving one data consumer. We discuss extensions to serve multiple consumers in Section~\ref{sec:conclusion}. Furthermore, the design supports a centralized model where providers' data is hosted by the arbiter and the arbiter is in charge of running computation. And it also supports a federated model, where providers' data stay with them and they are responsible for running computation on their own data.

\textbf{Data Consumer} has a machine learning model that needs to be trained. Let $f_w$ denote the model parameterized by $w \in \mathbb{R}^d$.
The data consumer provides a loss function $l$ that can be used to find a suitable model parameter $w$ by minimizing the loss $l(w; \mathcal{D})$, where $\mathcal{D}$ is a data set. The loss function is shared with the arbiter and, in the federated scenario, also with data providers.
In addition to the loss function, the data consumer also provides a utility function $U : \mathbb{R}^d \mapsto [0,1]$ that can be used to measure the quality of the model $f_w$. 
For example, the utility function can measure the quality of the produced model on a hold-out dataset that the data consumer has. We assume that the utility function is shared with the market arbiter but not with the data providers. 
Finally, we assume that the data consumer has a budget $B$, which is the total number of model updates that they can afford. Note that each model update costs the same amount of budget. We discuss extensions at the end of the section.

\textbf{Data Providers} are the market participants that have data that can be used to train machine learning models.
We assume that there are $n$ data providers.
We do not make assumptions on the type of data the providers have. Instead, we assume that each data provider has a \emph{model update oracle}. The $i$th data provider has an oracle function $\mathcal{O}_i:\mathbb{R}^d \mapsto \mathbb{R}^d$
that maps a model parameter $w$ to an update $g_i = \mathcal{O}_i(w)$. 
Each time a data provider uses its model update oracle---locally in the federated scenario, or remotely, through an  arbiter, in the centralized scenario---it will be compensated by one unit out of the total budget.
There can be many options of $\mathcal{O}_i$. For example, $\mathcal{O}_i(w)$ might be the negative full gradient or stochastic gradient of the loss function evaluated at $w$ by using the data of provider $i$; or it can be $w^{+}_i-w$, where $w^{+}_i$ is obtained by running several steps of stochastic gradient descent (SGD) starting from $w$; furthermore, the provider $i$ may also choose some privacy protection techniques when computing $\mathcal{O}_i(w)$, such as differential privacy~\citep{dwork2014algorithmic,chaudhuri2011differentially}. We leave the freedom of choosing $\mathcal{O}_i$'s to the data providers. Different choices of $\mathcal{O}_i$ will involve different levels of computational cost and privacy leakage risk. At one extreme, the data provider can return a random vector in $\mathbb{R}^d$ and thus induces little computational cost and privacy leakage.
However, the design of our algorithm encourages the data provider to provide high-quality model update.

\begin{algorithm*}[t!]
\caption{Adaptive Sampling with Online Stochastic Mirror Descent (OSMD) Sampler}
\label{alg:OSMD-sampler}
\begin{algorithmic}[1]
\STATE {\bfseries Input:} OSMD learning rate $\eta$; optimization stepsizes $\{\gamma_t\}^{T-1}_{t=0}$; parameter $\alpha \in [0,1]$, $\mathcal{A}=\mathcal{P}_{n-1}\cap[\alpha/n,\infty)^n$; total budget $B$, batch size $K$, and the number of communication rounds $T=\lfloor B/K \rfloor$; the initialization sampling distribution $p^0$ and model parameter $w^0$.
\STATE {\bfseries Output:} Trained model parameter $\hat{w}$ and number of access for each data provider $N^T(i)$ for all $i \in [n]$.
\STATE Initialize $N^0(i)=0$ for all $i \in [n]$.
\FOR{$t=0,\ldots,T-1$}{
\STATE Sample a batch of data providers $S^t=\{i^{t,0},\ldots,i^{t,K-1}\} \in [n]^K$ with replacement by sampling distribution $p^t$.
\STATE For each chosen data provider $i \in S^t$, broadcast the current model parameter $w^t$ to it and receive the model update from the provider's model update oracle $g^t_i=\mathcal{O}_i (w^t)$.
\STATE Let $N^{t+1}(i)=N^t(i)+\sum^{K-1}_{k=0}\mathds{1}\{i^{t,k}=i\}$ for all $i \in [n]$.
\STATE Let $\hat{u}^t=(\hat{u}^t_1,\ldots,\hat{u}^t_n)^{\top}$, where
\begin{equation*}
\hat{u}^t_i = \frac{ \sum^{K-1}_{k=0} \mathds{1} \left\{ i^{t,k} = i \right\} }{K p^t_i} \times \left( U \left( w^t + \gamma^t g^t_i \right) - U \left( w^t \right) \right)
\end{equation*}
\STATE Solve $p^{t+1} = \argmin_{q \in \mathcal{A}} - \eta \langle q, \hat{u}^t \rangle + D_{\Phi} (q \, \Vert \, p^t ) $ by Algorithm~\ref{alg:OSMD-sampler-solver}.
\STATE Update $w^t$ by $w^{t+1} = w^t + (\gamma^t/K) \sum_{i \in S^t} g^t_i$.
}
\ENDFOR
\STATE Set $\hat{w}=w^{T}$. 
\end{algorithmic}
\end{algorithm*}

\begin{algorithm*}[!t]
\caption{OSMD Solver}
\label{alg:OSMD-sampler-solver}
\begin{algorithmic}[1]
\STATE {\bfseries Input:} $p^t$, $\hat{u}^t$, $\eta$, $\alpha\in[0,1]$, and $\mathcal{A}=\mathcal{P}_{n-1}\cap[\alpha/n,\infty)^M$.
\STATE {\bfseries Output:} $p^{t+1}$.
\STATE Let $\tilde{p}^{t+1}_i = p^t_i \exp \left( \eta \hat{u}^t_i \right)$ for $i \in [n]$.
\STATE Sort $\{ \tilde{p}^{t+1}_i \}^n_{i=1}$ in a non-decreasing order: $\tilde{p}^{t+1}_{\pi(1)}\leq\ldots\leq\tilde{p}^{t+1}_{\pi(n)}$.
\STATE Let $v_i=\tilde{p}^{t+1}_{\pi(i)} \left( 1 -  \frac{i-1}{n} \alpha \right)$ for $i \in [n]$.
\STATE Let $u_i=\frac{\alpha}{n} \sum^n_{j=i} \tilde{p}^{t+1}_{\pi(j)}$ for $i \in [n]$.
\STATE Find the smallest $i$ such that $v_i > u_i$, denoted as $i^t_{\star}$.
\STATE Let $p^{t+1}_i =
\begin{cases}
\alpha / n & \text{if } \pi^{-1}(i) < i^t_{\star} \\
\frac{ \left((1 - ((i^t_{\star}-1) / n) \alpha ) \tilde{p}^{t+1}_{i}\right) }{  \left(\sum^n_{j=i^t_{\star}} \tilde{p}^{t+1}_{\pi(j)}\right) } & \text{otherwise}.
\end{cases}
$
\end{algorithmic}
\end{algorithm*}

\textbf{Market Arbiter} maintains an iterative training process until the consumer's total budget is exhausted. 
Let $K$ be the number of providers chosen at each round and we discuss the choice of $K$ later in Section~\ref{sec:thm-budget}.
Then the training process proceeds 
in $T=\lfloor B / K \rfloor $ iterations.
Let $w^0$ and $p^0=(1/n,\ldots,1/n)^{\top}$ be the initial model parameter and sampling distribution, respectively. 
For iterations $t=0,\ldots,T-1$, the market arbiter repeats the following process: 
(i) the arbiter uses the sampling distribution $p^t$ to sample with replacement a set, $S^t \in [n]^K$, of $K$ data providers;\footnote{We choose sampling with replacement to obtain a clearer budget allocation analysis. The key intuition is that when using sampling with replacement, the providers in $S^t$ are pairwise independent, allowing us to analyze them independently. In contrast, if we use sampling without replacement, we need to consider $S^t$ as a whole, which is overly complicated given that there are ${n \choose K}$ total possible subsets.} 
(ii) the arbiter broadcasts the current model parameter $w^t$ to the $i$th provider if $i \in S^t$;\footnote{Here broadcast is a local operation in the central setting where providers' data is directly accessible to the market, and a distributed operation in the federated setting.} 
(iii) the $i$th provider, $i \in S^t$, computes the model update $g^t_i = \mathcal{O}_i(w^t)$;
(iv) the arbiter updates the sampling distribution based on the received marginal utilities $\{U(w^t+\gamma^t g^t_i)\}_{i \in S^t}$ to obtain $p^{t+1}$, where $\gamma^t$ is the optimization stepsize; 
(v) the arbiter updates the model parameter as $w^{t+1} \leftarrow w^t + (\gamma^t / K) \sum_{i \in S^t} g^t_i$.
Algorithm~\ref{alg:OSMD-sampler} details the process.

We detail the step (iv) in the next section. 
The training process executed the market arbiter is the same as the training process in federated learning~\citep{kairouz2019advance} if $p^{t+1}=p^t$ in step (iv).
Thus our method can be applied to the setting where providers may have privacy concerns or are physically distributed. 

Finally, we note that each access costs the same fixed amount of budget. This makes sense because in the context of ML training, a single update does not have a major impact on the result. However, it is conceivable to introduce pricing techniques to price each access differently. We do not deal with this scenario in this paper.

\section{Online Stochastic Mirror Descent Sampler}
\label{sec:OSMD}

We detail step (iv) of Algorithm~\ref{alg:OSMD-sampler} where the sampling distribution is updated. The sampling distribution is a crucial ingredient in the design of Algorithm~\ref{alg:OSMD-sampler} as it connects to budget and revenue allocation. 
We start by examining the expected utility gain in the round $t$. 
We would like to choose the set of providers $S^t$ to maximize the utility gain 
\begin{multline*}
\Delta U^t \coloneqq U \left( w^{t+1} \right) - U \left( w^t \right) \\
= U \left( w^t + \frac{\gamma^t}{K} \sum_{i \in S^t} g^t_i \right) - U \left( w^t \right).
\end{multline*}
Maximizing the utility gain over $S^t$ is hard as it is a combinatorial problem. More specifically, the utility gain of selecting one provider depends not only on the provider itself but also on the other providers chosen in $S^t$. As a result, one needs to consider each subset of $[n]$, leading to a total of ${n \choose K}$ choices. The size of this space induces significant challenge.

To overcome this issue, one can consider the surrogate loss defined by \emph{cumulative marginal utility gain}:
\begin{equation}
\label{eq:cum-marg-util}
\Delta U^t_\text{cm} \coloneqq \sum_{i \in S^t} U \left( w^t + \gamma^t g^t_i \right) - U \left( w^t \right).
\end{equation}
The advantage of marginal utility is that when selecting one provider, the utility gain depends solely on that provider and is independent of other providers. In this way, we can treat each provider independently, resulting in a total of $n$ choices, which is much smaller than ${n \choose K}$ for $K > 1$. The surrogate loss simplifies the analysis and also leads to more efficient algorithm design. When $K=1$, we have $\Delta U^t_\text{cm}=\Delta U^t$; when $K > 1$, a larger $\Delta U^t_\text{cm}$ also generally implies a larger~$\Delta U^t$.

Let $\Delta U^t_i = U ( w^t + \gamma^t g^t_i ) - U ( w^t )$ and $p^t=(p^t_1,\ldots,p^t_n)^{\top}$. Recall that the set $S^t$ is obtained by sampling with replacement from $p^t$ and, thus, we have
\begin{equation*}
\mathbb{E}_{S^t} \left[  \Delta U^t_\text{cm} \right] = K \sum^n_{i=1} p^t_i \Delta U^t_i,
\end{equation*}
where $\mathbb{E}_{S^t}[\cdot]$ denotes the expectation taken with respect to $S^t$. Therefore, the distribution $p^t$ can be chosen to maximize $\mathbb{E}_{S^t} [  \Delta U^t_\text{cm} ]$ when $\Delta U^t_1,\ldots,\Delta U^t_n$ are known.

In practice, obtaining $\Delta U^t_1,\ldots,\Delta U^t_n$ requires access to the model update oracle of each provider, which is too costly. 
Furthermore, since $S^t$ is sampled based on $p^t$, the information will not be revealed until we make a decision on the sampling distribution.
In each round, the arbiter has a partial information history $\{\{ \Delta U^l_i \}_{i \in S^l}\}^{t-1}_{l=0}$ and $\{p^l\}^{t-1}_{l=0}$, based on which it updates the sampling distribution. We do not make any assumptions on how the utility is generated; thus it can be dependent on or even adversarial to the arbiter's decision.

The sampling process can be viewed as an online learning task with partial information feedback, where a game is played between the arbiter and the providers.
The goal of the online learning task is to choose a sequence of sampling distributions to maximize $\mathbb{E} [ \sum^{T-1}_{t=0} \Delta U^t_\text{cm} ]$.
A key concern about this online learning task is the non-stationarity of the utility process $\{ \Delta U^t_i \}_{t \geq 1}$, $i \in [n]$. More specifically, data providers that are useful in the early stage of the training process are not necessarily as useful in the later stage. To understand why, suppose that the model parameter is randomly initialized and far away from a good model parameter. At this stage, even a noisy data set can help improve the model by providing a rough guidance --- the additional advantage of a high-quality data set may not be obvious at this stage as improving the model is easy. However, as the training proceeds, the model will achieve reasonably good performance, and further improvement of the model will require higher quality data --- the difference between the utilities of valuable and noisy providers now becomes more distinctive.

Based on the above discussion, we propose to use online stochastic mirror descent (OSMD)~\citep[Chapter 31]{lattimore2020bandit} to make update of the sampling distribution.
Let $D_{\Phi} \left(x \,\Vert\, y\right) = \Phi (x) - \Phi (y) - \langle \nabla \Phi (y), x - y \rangle$ for $x,y \in \mathbb{R}^n_{+}$ be the Bregman divergence, where $\Phi(x)=\sum^n_{i=1} x_i \log x_i $  with $0\log0$ defined as $0$, is the unnormalized negative entropy.
Then the key updating rule is achieved by solving the following optimization problem:
\begin{equation}
\label{eq:OSMD-update}
p^{t+1} = \argmin_{q \in \mathcal{A}} - \eta \langle q, \hat{u}^t \rangle + D_{\Phi} (q \, \Vert \, p^t )
\end{equation}
where $p^t$ is the current sampling distribution, $\hat{u}^t$ is an unbiased estimate of the utility vectors $(\Delta U^t_1,\ldots, \Delta U^t_n)^{\top}$, and $\eta$ is a chosen learning rate.
We restrict the sampling distribution in a clipped simplex $\mathcal{A}=\mathcal{P}_{n-1}\cap[\alpha/n,\infty)^n$, where $\alpha \in [0,1]$ is a tuning parameter. 
By restricting the sampling distribution to be away from zero, we avoid the algorithm from committing too hard on a single provider, which otherwise may prevent the discovery of change points in a non-stationary environment.
The first part of the objective encourages the sampling distribution to maximize the most recent utility gain, while the second part of the objective prevents it from deviating too far away from the previous sampling distribution.  The detailed algorithm is described in Algorithm~\ref{alg:OSMD-sampler}. The problem in \eqref{eq:OSMD-update} can be solved efficiently by Algorithm~\ref{alg:OSMD-sampler-solver} (Proposition~\ref{prop:OSMD-solver} in Appendix~\ref{sec:proof-prop-1}). The most costly step of Algorithm~\ref{alg:OSMD-sampler-solver} is sorting $\{ \tilde{p}^{t+1}_i \}^n_{i=1}$, which has a computational complexity of $O(n\log n)$. However,
noting that most of the entries in $\tilde{p}^{t+1}$ are sorted in the previous round, we may use an adaptive sorting algorithm to achieve a much faster running time~\citep{estivill1992survey}.

\subsection{Analysis of Budget Allocation}
\label{sec:thm-budget}

We provide theoretical guarantees on budget allocation in the form of a regret bound when comparing against an oracle competitor. 
At the beginning of round $t$, we allow the competitor to know the provider with the largest $\Delta U^t_i$; thus, to maximize $\Delta U^t_\text{cm}$, the competitor will choose it $K$ times in round $t$.
Due to the non-stationary nature of the problem, we also allow the oracle to change its decision from iteration to iteration, with the total number of changes not exceeding $m-1$ times, where $m$ is a chosen tuning parameter. Note that with a larger $m$, we are competing against a more flexible and harder oracle.

For $m \geq 1$, the action space $\Gamma_m$ is defined as 
\begin{multline*}
\Gamma_m \coloneqq \left\{ \left( a^0,\ldots,a^{T-1} \right) \in [n]^T : \right. \\
\left. \sum^{T-1}_{t=1} \mathds{1} \left\{  a^t \neq a^{t-1} \right\} \leq m - 1 \right\}.
\end{multline*}
The space $\Gamma_m$ contains all action sequences where the action changes at most $m-1$ times. 
This space contains the action sequence of an oracle that knows the identities of the top $m$ providers and the right order to access them. 
We define regret of the budget allocation of Algorithm~\ref{alg:OSMD-sampler} as
\begin{equation}
\label{eq:regret}
R_B \coloneqq \mathbb{E} \left[ \max_{ (a^t) \in \Gamma_m } K \sum^{T-1}_{t=0} \Delta U^t_{a^t} -  \sum^{T-1}_{t=0} \Delta U^t_\text{cm} \right],
\end{equation}
where $\Delta U^t_\text{cm}$ is defined in~\eqref{eq:cum-marg-util} and the expectation is taken with respect to all sources of randomness. 
From the consumer's perspective, a larger utility gain means a better budget allocation and, therefore, regret $R_B$ measures the loss of budget due to the lack of quality information about providers. 
We have the following guarantee on the regret of Algorithm~\ref{alg:OSMD-sampler}.
\begin{theorem}
\label{thm:regret-bound}
Assume $B=K T$ and $U(\cdot) \in [0, 1]$.
Let $\{p^t\}_{t=0}^{T-1}$ be the sequence of sampling distribution generated by Algorithm~\ref{alg:OSMD-sampler}.
The regret $R_B$ defined in~\eqref{eq:regret} is bounded by
\begin{equation*}
R_B \leq \alpha B  + \frac{\eta n B}{2} + \frac{m K \log(n/\alpha)}{\eta}.
\end{equation*}
\end{theorem}

If $\alpha=\sqrt{n/B}$ and $\eta=\sqrt{m \log(n B)/(n T)}$, it follows that $R_B / B = O\left( \sqrt{n m K \log (n B) / B} \right)$.
Treating $n$, $K$, $m$ as constants, we have that the asymptotic average regret is zero, $R_B/B=o(1)$, which is used as a standard to measure budget allocation~\citep{agarwal2019marketplace}. By Theorem~\ref{thm:regret-bound}, we see that a larger $K$ will induce a larger regret within the same budget; however, a larger $K$ is helpful in reducing the variance of the model parameter. Thus, a good choice of $K$ should maintain a trade-off between these two concerns.
Note that we do not provide specific quantification of how the variance depends on $K$. The reason is that the variance of the model parameter depends on numerous factors, such as the loss function, the learning rate, and the model update oracle. We intentionally leave these factors unspecified to maintain generality. That being said, a simple example for reference is the use of mini-batch SGD or local SGD as the model update oracle. In this case, the convergence analysis from~\citep{karimireddy2020scaffold} can be adopted. It then becomes evident that the number of providers we select in each round will influence the convergence speed through the variance.

\subsection{Analysis of Revenue Allocation}
\label{sec:thm-rev-alloc}

Data providers furnish their data with the expectation of receiving compensation, thus the equitable distribution of revenue among providers is a crucial consideration for data markets. This differs from the budget allocation issue, which we framed as an online optimization problem. Revenue allocation, on the other hand, falls within the scope of a "fair division" problem ~\citep{wiki:fairdivision}. Generally, there is no universally acknowledged standard for resolving revenue division dilemmas, as opposed to online optimization scenarios. A prevalent strategy is to first establish each participant's "entitlement" and subsequently distribute the revenue in accordance with this assigned entitlement.

Existing research primarily applies the Shapley value~\citep{shapely1953value} to determine entitlement, given its attributes that are considered essential for equitable allocation in cooperative games. The Shapley value is widely acknowledged to singularly satisfy four axioms of a fair reward distribution in cooperative game theory, namely: symmetry, efficiency, null player, and linearity~\citep{shapely1953value,kwon2022beta}. Nevertheless, the computation of the Shapley value is considerably resource-intensive. Consequently, the total funds available for allocation are reduced by the computational expense incurred in deriving the Shapley value itself. Furthermore, the relevance and necessity of these axioms for revenue allocation within real-world data markets remain a matter of debate.

In this paper, we introduce a novel concept of entitlement referred to as the \emph{number of accesses to the model update oracle}. By correlating the compensation a data provider receives with the quantity of model update oracles it contributes, the compensation aligns with the level of service provided. We posit that this methodology offers a fresh perspective on the study of revenue allocation.

Given that our model does not frame revenue allocation as a cooperative game, our approach diverges from the typical analytical framework associated with the Shapley value. Nevertheless, we demonstrate that our revenue allocation method possesses certain attributes that echo the underlying principles of the Shapley value's axioms.
\begin{itemize}
    \item Symmetry: The symmetry axiom basically requires the evaluation result is invariant to random shuffling of the providers. Note that by the design of Algorithm~\ref{alg:OSMD-sampler}, it is clear that the distribution of $N^T(i)$ for any $i \in [n]$ will not change by shuffling the providers, thus our revenue allocation approach enjoys the symmetry property.
    \item Efficiency: Efficiency axiom requires that the sum of values of all providers should be equal to the value of the whole dataset. In our case, by the definition of $N^T(i)$, we have $\sum^n_{i=1}N^T(i)$ equal the budget for training. Thus, our revenue allocation approach also shares certain efficiency property.
    \item Null Player: This axiom requires to assign zero value to data providers that make no contribution. In our case, any data provider has a chance to get compensation no matter how much contribution it makes. However, our revenue allocation approach ensures that providers making no contribution will get asymptotically minimum compensation. To see this, assuming that there exists an $i \in [n]$, such that $U(w +\gamma\mathcal{O}_i(w)) \leq U(w)$ for any $w \in \mathbb{R}^d$ and $\gamma>0$ (which indicates that provider $i$ cannot make positive contribution), then by the design of Algorithm~\ref{alg:OSMD-sampler}, we will have $p^1_i\geq p^2_i \geq \ldots \geq \alpha/n$. This means that the probability of accessing provider $i$ is asymptotically decreasing to the lower bound. When $\alpha$ is small, the compensation received by provider $i$ will be negligible. 
\end{itemize}

By the intrinsic design of our algorithm, it is hard for our revenue allocation approach to satisfy the linearity axiom; however, it is unclear why linearity matters in practical applications.
In fact, previous research also tries to remove some of the four axioms for a more practical solution. For example, \citet{kwon2022beta} removes the efficiency axiom and proposes a variant of Shapley value which has better performance on many practical machine learning tasks.
Here we relax the linearity axiom for a similar reason.

\section{Evaluation}
\label{sec:exp}

We evaluate our method empirically on both real-world data and synthetic data experiments. 
We concentrate on three main questions. 
First, for revenue allocation, we ask whether our method provides similar allocation quality as the state-of-the-art methods based on Shapley value. For this purpose, we choose two Shapley value based methods as the baseline methods: Gradient Shapley (G-Shapley)~\citep{ghorbani2019data} and Federated Shapley value with permutation sampling estimation (FedSV-PS)~\citep{wang2020principled}. These two methods do not require re-training of the model multiple times, as required by most of the Shapley value based methods, and are suitable for modern large-scale machine learning models like Deep Neural Networks (DNN).
Second, for budget allocation, we ask whether our method delivers better performing model than the standard training procedure based on uniform sampling.
Finally, we study the runtime of our method. We compare it with uniform sampling, G-Shapley and FedSV-PS on real-world data sets.

\subsection{Main Results}
\label{sec:real-data-exp}

\begin{figure}[t!]
\centering
\includegraphics[width=\columnwidth]{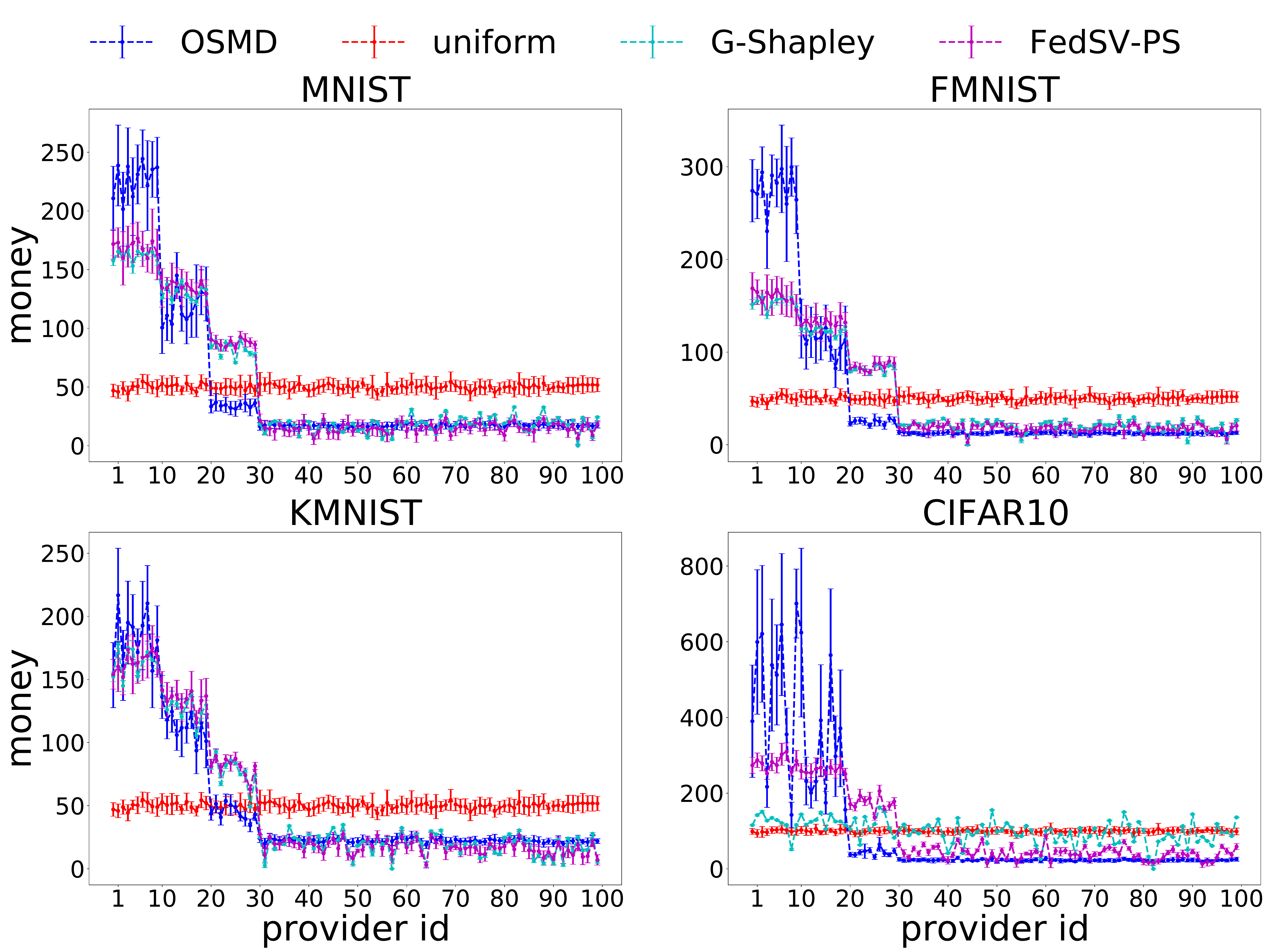}
\caption{Revenue allocation result of OSMD (our method), uniform sampling, Gradient Shapley (G-Shapley) and Federated Shapley value with permutation sampling estimation (FedSV-PS). The solid dot denotes the mean and the error bar denotes the mean $\pm$ standard deviation.}
\label{fig:real-money}
\end{figure}

We implement the experiments on four image classification task data sets: MNIST~\citep{lecun-mnisthandwrittendigit-2010}, Fashion-MNIST (FMNIST)~\citep{xiao2017fashion}, Kuzushiji-MNIST (KMNIST)~\citep{clanuwat2018deep}, and CIFAR-10~\citep{krizhevsky2009learning}. We create $100$ data providers with each provider having $400$ training samples. We then use another $1000$ samples from the training set as the hold-out set to compute the utility, which is defined as the negative loss on the hold-out set. Finally, we use the testing set which contains $10,000$ samples to measure the accuracy of the model. 
See Appendix~\ref{sec:detail-real-world-exp} for more details about data preparation.
For MNIST, FMNIST and KMNIST, we train a two-layer neural network with one-hidden layer of $300$ neurons; for CIFAR-10, we train a Convolutional Neural Network (CNN) with one convolutional layer and two feed-forward layers. Each method is repeated $10$ times with different random seeds. 

To create quality discrepancies across providers, we corrupt $\beta\%$ of the providers' samples by changing a sample's correct label to a randomly chosen wrong label. For provider $1$--$10$, $\beta=0$; for provider $11$--$20$, $\beta=20$; for provider $21$--$30$, $\beta=50$; for the rest of the providers, $\beta=90$. Additional details are provided in Appendix~\ref{sec:detail-real-world-exp}.

\textbf{Revenue Allocation:} We first examine whether our method achieves similar revenue allocation results as Shapley-based mehtods (we will show later it solves the problem much faster). For our method and uniform sampling, the revenue is allocated proportional to the number of accesses, while for G-Shapley and FedSV-PS, the revenue is allocated proportional to the positive part of Shapley value, defined as $\max\{\text{SV},0\}$, where $\text{SV}$ denotes the Shapley value. The result is shown in Figure~\ref{fig:real-money}. We see that OSMD has a similar revenue allocation performance as G-Shapley and FedSV-PS on MNIST, KMNIST and FMNIST. OSMD tends to allocate more revenue on the most valuable providers while G-Shapley and FedSV-PS seem to keep a balance between different quality providers. Interestingly, we see that G-Shapley underperforms on CIFAR-10. This might be because the task is difficult, and with a high portion of data being corrupted, the one-step gradient approximation does not work well. The take-away message is that the new method's allocation closely resembles that of Shapley-based methods.

\textbf{Budget Allocation:} We study the performance of budget allocation. We show the model accuracy on the hold-out test set along the training process. The results in Figure~\ref{fig:real-test-accu} show that OSMD successfully helps the model obtain a better performance than uniform sampling by identifying the valuable providers and adjusting the sampling distribution accordingly: the new algorithm uses the budget efficiently.

\textbf{Computational Cost:} We also compare the runtime (wall clock time) of the four methods mentioned above. Higher runtimes imply higher costs on the arbiter side, thus reducing runtime is a crucial concern in designing practical data market, as the higher the costs, the higher the proportion of budget that needs to be used to cover those costs and thus the worse the models delivered to consumers. To measure the scalability of different methods, we fix the number of samples that each provider has to $100$ and increase the number of providers from $50$ to $400$. The result is shown in Figure~\ref{fig:real-time-used}. Compared with G-Shapley and FedSV-PS, our method is much faster, only marginally more costly than uniform sampling and highly scalable. This result indicates that our method is more suitable to modern ML applications where the size of both dataset and model is large.

\begin{figure}[t!]
\centering
\includegraphics[width=\columnwidth]{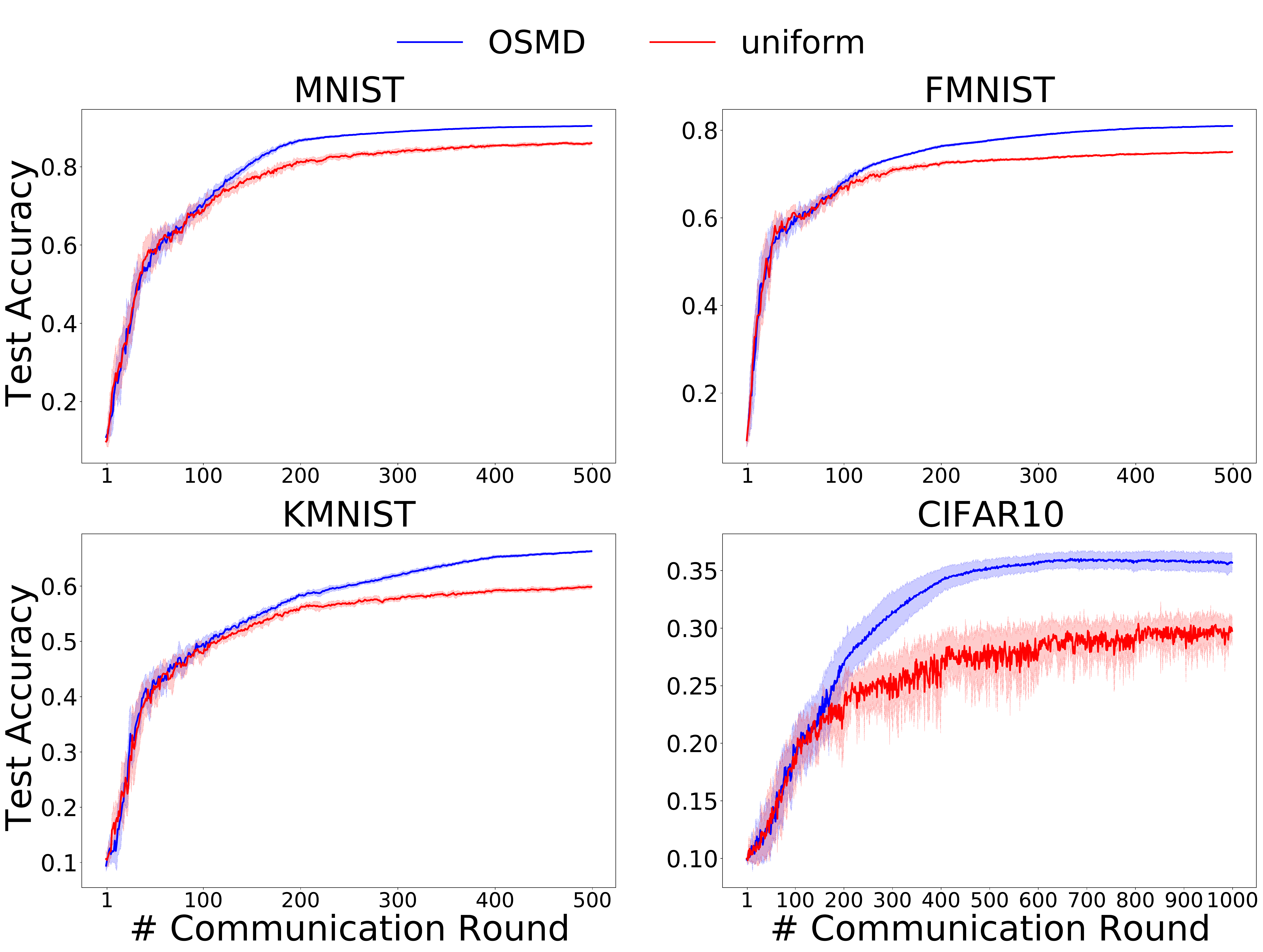}
\caption{Budget allocation result of OSMD (our method) and uniform sampling. The solid line denotes the mean and the shadow region denotes the mean $\pm$ standard deviation.}
\label{fig:real-test-accu}
\end{figure}

\begin{figure}[t!]
\centering
\includegraphics[width=\columnwidth]{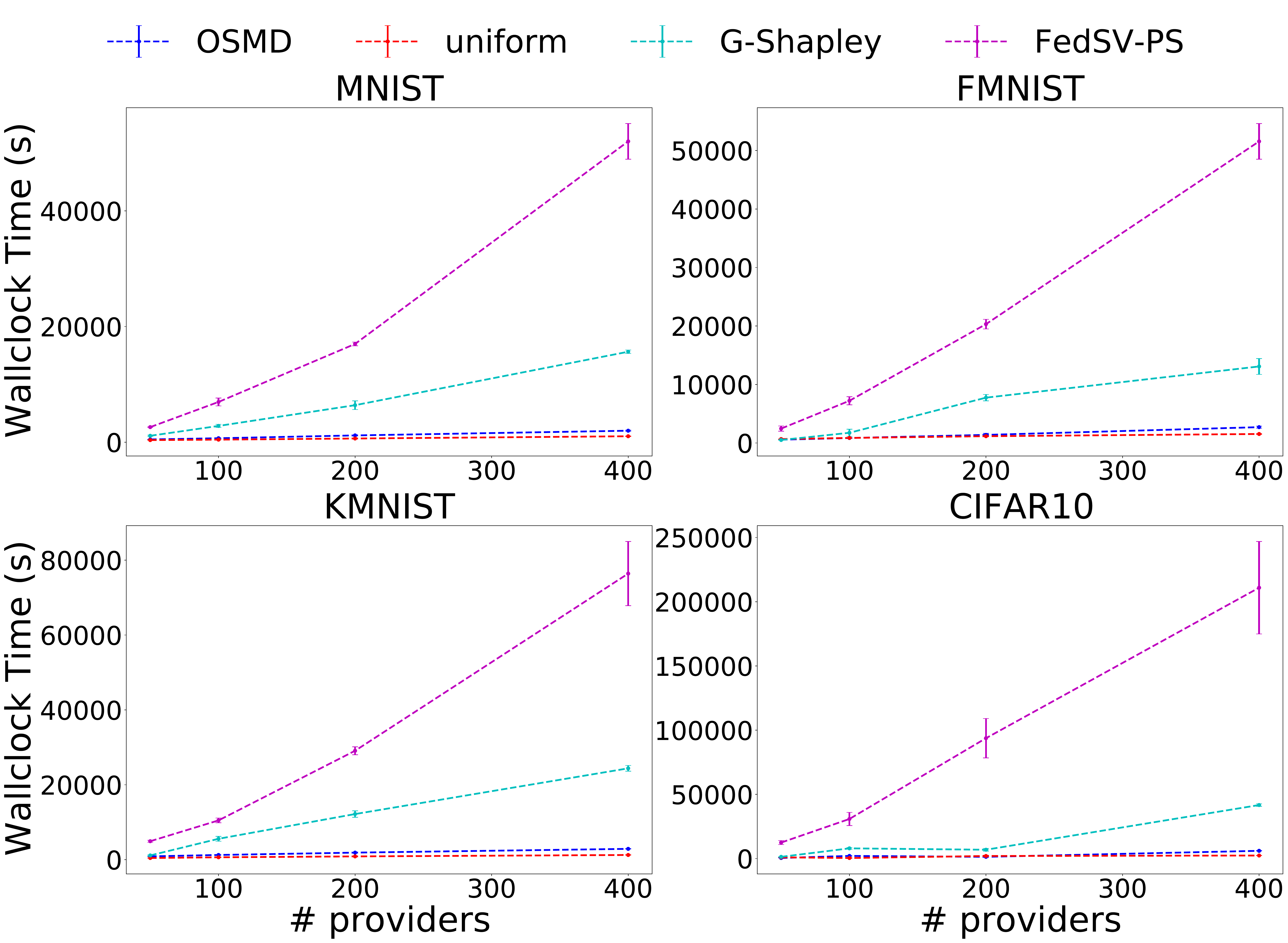}
\caption{Runtime of OSMD (our method), uniform sampling, Gradient Shapley (G-Shapley) and Federated Shapley value with permutation sampling estimation (FedSV-PS). The solid dot denotes the mean and the error bar denotes the mean $\pm$ standard deviation.}
\label{fig:real-time-used}
\end{figure}

\subsection{Mixture Linear Regression}
\label{sec:mixture-linear-reg}

In the previous section, the difference of data quality between providers is created by corrupting the label. In this section, we present a different approach to create such difference by creating distribution shift of $\mathbb{P}(y \mid x)$.
More specifically, we implement a mixture linear regression experiment.
Assume that data consumer wants to train a linear regression model. Both data consumer's and data providers' features are drawn from the same distribution $\mathbb{P}_X$; however, the response variable is generated by the following mixture linear regression model:
\begin{align*}
& x_{ij} \in \mathbb{R}^d \sim \mathbb{P}_X, \quad  j = 1, \ldots, m_i, i=0,1,\ldots,n, \\
& z_0,z_1,\ldots,z_n \sim {\rm uniform}([K]), \\
& y_{ij} = \langle w_{z_i}, x_{ij} \rangle + \varepsilon_{ij}, \, \varepsilon_{ij} \overset{i.i.d.}{\sim} N(0,0.5^2), \\
& \qquad \qquad\qquad \qquad j = 1, \ldots, m_i,\ i=0,1,\ldots,n,
\end{align*}
where $z_0$ is the latent group label of the consumer, and $z_i$ with $i \geq 1$ is the latent group label of data provider $i$. 
The set of regression parameters is $\{w_1,\ldots,w_K\}$.
This way, data providers with $z_i=z_0$ should have the same true parameter as the data consumer, thus is more valuable.
We set $K=4$ and $z_0=1$. 
We let $d=1000$ and generate each entry of $w_k$ i.i.d.~from $U[0.5(k-1), 0.5 k]$ for $k=1,\ldots,4$. In addition, we generate entries of $x_{ij}$ i.i.d.~from $N(0,1)$ and
we initialize the parameter from $(-1,\ldots,-1)^{\top}$.
See details in Appendix~\ref{sec:detail-mixture-linear-exp}.

We compare the budget allocation and revenue allocation of our method with other baselines.
The results are shown in Figure~\ref{fig:mix-linear-reg}. Regarding budget allocation, we see that both OSMD and FedSV-PS generate fair revenue allocation, by which we mean that data providers with smaller $\Vert w_{z_i} - w_{z_0} \Vert$ are compensated higher. Interestingly, we see that the allocation result of G-Shapley is opposite to the truth. We conjecture this is because providers with larger $\Vert w_{z_i} - w_{z_0} \Vert$ will make gradients evaluated at the initial point have larger norm, and thus achieve more utility gain at the initial stage of the training. This is another example in addition to the previous CIFAR-10 case where G-Shapley fails. Regarding budget allocation, we see that there is a clear phase transition. Initially, since providers from all groups help, the parameter estimates get closer to the true parameter. All providers seem to have similar values, thus OSMD and uniform sampling perform similarly. As the parameter estimation gets closer to the true parameter, there is an increasing difference between providers, and OSMD starts to tell apart the ``right'' providers from the ``wrong'' providers. This experiment demonstrates the point raised in Section~\ref{sec:OSMD} about non-stationarity: that providers whose data is useful at a point in time may not be useful later. The take-away message is that taking the non-stationarity of the problem into consideration is necessary for the algorithm to perform well.

\begin{figure}[t!]
\centering
\includegraphics[width=\columnwidth]{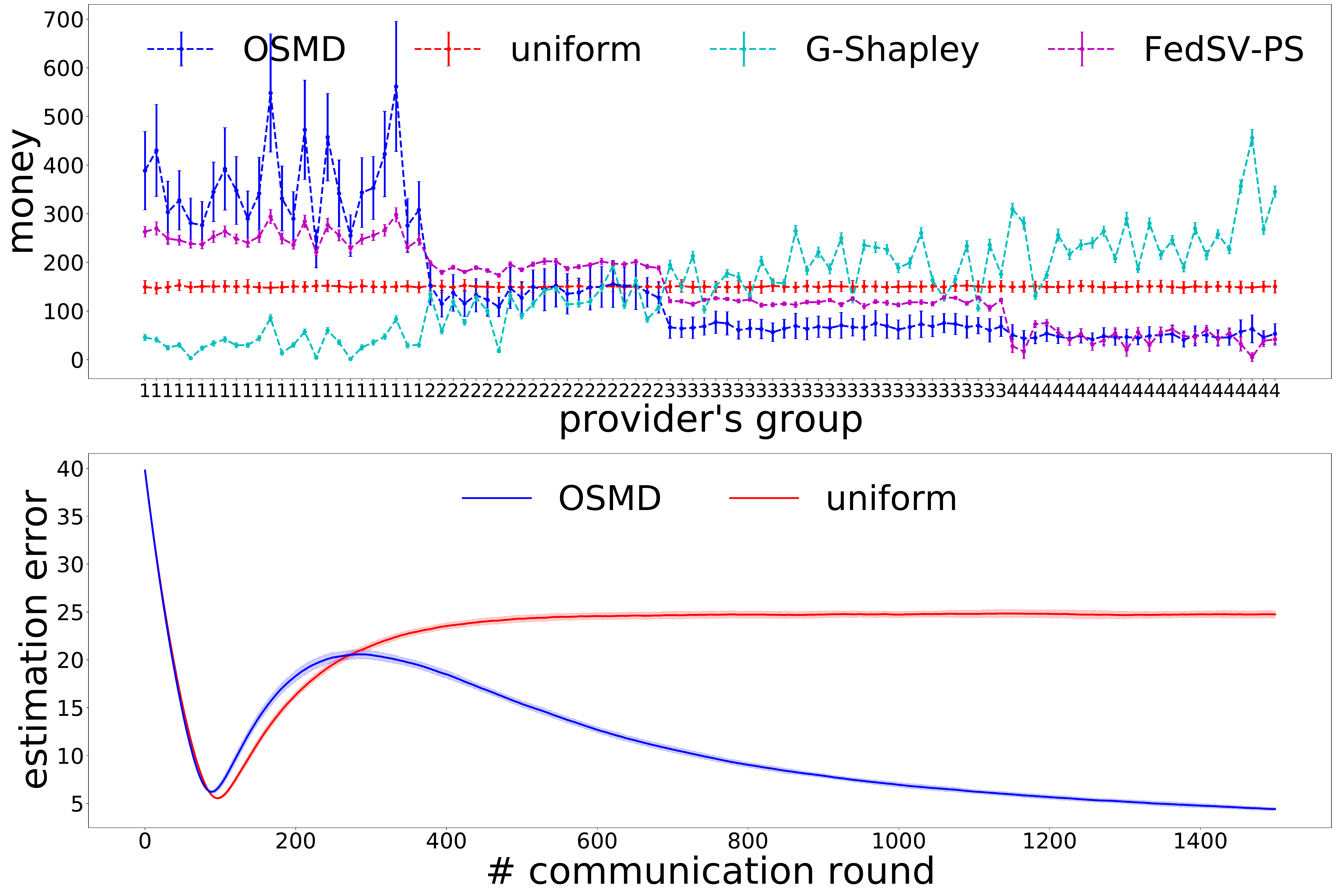}
\caption{The result of mixture linear regression experiment. The above figure shows the revenue allocation result. Note that the consumer's parameter is from group $1$. The below figure shows the budget allocation result.}
\label{fig:mix-linear-reg}
\vskip -0.1in
\end{figure}

\section{Conclusion}
\label{sec:conclusion}

We proposed a new algorithm based on adaptive sampling that solves the problems of budget allocation and revenue allocation in linear time, thus paving the way to the implementation of practical data markets. The new algorithm achieves state-of-the-art revenue allocation performance while ensuring efficient budget allocation performance at the same time. The key insight of the algorithm design is to open the black box of the training process, and use an adaptive sampling technique to encourage more valuable providers to make more contributions. We offer theoretical guarantees and an empirical evaluation of the algorithm.

There are many interesting future directions. 
First, the problem of pricing each query of model update remains open.
Second, while our algorithm can be easily extended to serve multiple consumers, it is worth noting that in practice, different consumers may cooperate or compete with each other. For example, the value of a machine learning model possessed by a company may not only depend on its own accuracy, but may also depend on the performance of the model possessed by the company's competitor. The study of such interactions between multiple consumers provides a fruitful research direction.
In addition, one may also consider the effect of strategic behavior, especially if providers can collude. 
Finally, combining differential privacy with our method can also be a promising direction.

\section*{Acknowledgements}

The research of MK is supported in part by NSF Grant ECCS-2216912.
This work was completed in part with resources provided by the University of Chicago Booth’s Mercury Computing Cluster.

\newpage

\bibliography{boxinz-papers}
\bibliographystyle{icml2023}

\newpage
\appendix
\onecolumn

\section{Details about Experiments}

\subsection{Details about Real-World Experiments}
\label{sec:detail-real-world-exp}

\textbf{Dataset:} We implement the experiments on four image classification task datasets: MNIST~\citep{lecun-mnisthandwrittendigit-2010}, Fashion-MNIST (FMNIST)~\citep{xiao2017fashion}, Kuzushiji-MNIST (KMNIST)~\citep{clanuwat2018deep}, and CIFAR-10~\citep{krizhevsky2009learning}. All datasets are composed of training set and testing set. The sample size information is summarized in Table~\ref{table:info-img-cls}. Both training set and testing set of all datasets are perfectly balanced, that is, each label contains the equal number of samples.
We create $100$ data providers with each provider having $400$ training samples. We then use $1000$ samples from the rest of the training set as the hold-out set to compute the utility, which is defined as the negative loss on the validation set. Finally, we use the testing set to measure the accuracy of the model.
\begin{table}[h]
\centering
\begin{tabular}{||c c c c||} 
\hline
Dataset & Sample Size of Training Set & Sample Size of Testing Set & Label Balance \\[0.5ex]
\hline\hline
MNIST & $60,000$ & $10,000$ &  Balanced \\ 
KMNIST & $60,000$ & $10,000$ &  Balanced \\
FMNIST & $60,000$ & $10,000$ &  Balanced \\
CIFAR10 & $50,000$ & $10,000$ &  Balanced \\
\hline
\end{tabular}
\caption{The information about image classification datasets.}
\label{table:info-img-cls}
\end{table}

\textbf{Machine Learning Model:} For MNIST, FMNIST and KMNIST, we train a two-layer neural network with one-hidden layer of $300$ neurons; for CIFAR-10, we train a Convolutional Neural Network (CNN) with one convolutional layer and two feed-forward layers. We use the cross-entropy loss as the loss function. 

\textbf{Model Update Oracle:} For a given model parameter $w$, we choose $\mathcal{O}_i(w)=w^+_i - w$, where $w^+_i$ is obtained by running mini-batch SGD on the provider $i$'s data for one epoch. More specifically, we divide the provider's data into mini batches, that is, $\mathcal{D}_i = \cup^L_{l=1} \mathcal{B}_l $, and then let $w^0=w$, $w^{l+1} = w^l - (\gamma_{\text{local}}/B)\nabla l(w^l, \mathcal{B}_l)$ for $l=0,\ldots,L-1$, where $B$ is the size of the mini-batches and $\gamma_{\text{local}}$ is the local step size. Finally, we let $w^+_i=w^L$.

\textbf{Parameter Choice:} We choose the optimization stepsizes $\{\gamma_t\}^{T-1}_{t=0}$ in Algorithm~\ref{alg:OSMD-sampler} as $\gamma^t=\gamma_{\text{global}}=0.1$ for all $t=0,\ldots,T-1$. Besides, we let $\alpha=0.01$, and $\eta=1.0$. We set the total number of communication rounds between the arbiter and providers as $T=500$ for MNIST, FMNIST and KMNIST, and as $T=1000$ for CIFAR10. Furthermore, we have $K=0.1 n$. 
For G-Shapley, we choose the same convergence criterion as in~\citet{ghorbani2019data}; however, we set the back check length to be $30$ for CIFAR-10 and to be $10$ for the rest of tasks, and we set the convergence threshold as $0.01$. For FedSV-PS, we use the same convergence criterion to decide the determination in each communication round. The back check length is $10$ and the convergence threshold is $0.05$. In addition, each method is repeated $10$ times with different random seeds. 

\textbf{Corruption:} To create quality discrepancies across providers, we corrupt $\beta\%$ of the providers' samples by changing a sample's correct label to a randomly chosen wrong label. For provider $1$--$10$, $\beta=0$; for provider $11$--$20$, $\beta=20$; for provider $21$--$30$, $\beta=50$; for the rest of the providers, $\beta=90$.

\textbf{Computational Cost Experiment:} To compare the computational cost of our method with other competitors, we fix the number of training samples possessed by each provider as $100$, and set the number of providers as $50$, $100$, $200$ and $400$, and record the runtime.

\subsection{Details about Mixture Linear Regression Experiments}
\label{sec:detail-mixture-linear-exp}

\textbf{Data Generation Process:} Both data consumer's and data providers' features are drawn from the same distribution $\mathbb{P}_X$; however, the response variable is generated by the following mixture linear regression model:
\begin{align*}
& x_{ij} \in \mathbb{R}^d \sim \mathbb{P}_X \quad \text{ for all } j = 1, \ldots, m_i, i=0,1,\ldots,n, \\
& z_0,z_1,\ldots,z_n \sim [K] \text{ uniformly random}, \\
& y_{ij} = \langle w_{z_i}, x_{ij} \rangle + \varepsilon_{ij}, \, \varepsilon_{ij} \overset{i.i.d.}{\sim} N(0,0.5^2) \quad \text{ for all } j = 1, \ldots, m_i, i=0,1,\ldots,n,
\end{align*}
where $z_0$ is the latent group label of the consumer, and $z_i$ with $i \geq 1$ is the latent group label of data provider $i$. 
The set of regression parameters is $\{w_1,\ldots,w_K\}$.
This way, data providers with $z_i=z_0$ should have the same true parameter as the data consumer, thus is more valuable.
We set $K=4$ and $z_0=1$. Besides, we let $d=1000$ and generate each entry of $w_k$ i.i.d. from $U[0.5(k-1), 0.5 k]$ for $k=1,\ldots,4$. In addition, we generate entries of $x_{ij}$ i.i.d. from $N(0,1)$ and
we initialize the parameter from $(-1,\ldots,-1)^{\top}$. Given a parameter estimation $\hat{w}$, the estimation error is measured by $\Vert \hat{w} - w_1 \Vert$.

\textbf{Training Process:} For model update oracle, we let $\mathcal{O}_i(w)=-\nabla l (w, \mathcal{D}_i)$, where $\mathcal{D}_i=\{(x_{ij}, y_{ij})\}^{m_i}_{j=1}$ and 
\begin{equation*}
l (w, \mathcal{D}_i)=\frac{1}{2 m_i} \sum^{m_i}_{j=1} \left( y_{ij} - \langle w,x_{ij}\rangle \right)^2.
\end{equation*}
Besides, we choose the optimization stepsizes $\{\gamma_t\}^{T-1}_{t=0}$ in Algorithm~\ref{alg:OSMD-sampler} as $\gamma^t=\gamma_{\text{global}}=0.01$ for all $t=0,\ldots,T-1$. Besides, we let $\alpha=0.01$, and $\eta=0.001$. We set the total number of communication rounds between the arbiter and providers as $T=1500$. Furthermore, we have $K=0.1 n$. The parameter setting of G-Shapley and FedSV-PS are the same as in Section~\ref{sec:detail-real-world-exp}.

\section{Technical proofs}

\subsection{Efficient Solution to \eqref{eq:OSMD-update}}
\label{sec:proof-prop-1}

The following proposition justifies Algorithm~\ref{alg:OSMD-sampler-solver}.

\begin{proposition}
\label{prop:OSMD-solver}
Let 
\[
\tilde{p}^{t+1}_i = p^t_i \exp \left( \eta \hat{u}^t_i  \right), \qquad i \in [n].
\]
Let $\pi:[n]\mapsto[n]$ be a permutation such that $\tilde{p}^{t+1}_{\pi(1)}\leq\dots\leq\tilde{p}^{t+1}_{\pi(n)}$. Let $i^t_{\star}$ be the smallest integer $i$ such that
\begin{equation*}
\tilde{p}^{t+1}_{\pi(i)} \left( 1 -  \frac{i-1}{n} \alpha \right) > \frac{\alpha}{n} \sum^n_{j=i} \tilde{p}^{t+1}_{\pi(j)}.
\end{equation*}
Then the solution of~\eqref{eq:OSMD-update} is
\begin{equation*}
p^{t+1}_i =
\begin{cases}
\alpha / n & \text{if } \pi^{-1}(i) < i^t_{\star} \\
\left((1 - ((i^t_{\star}-1) / n) \alpha ) \tilde{p}^{t+1}_{i}\right)/\left(\sum^M_{j=i^t_{\star}} \tilde{p}^{t+1}_{\pi(j)}\right) & \text{otherwise}.
\end{cases}
\end{equation*}
\end{proposition}

\begin{proof}
    
First, we show that the solution of~\eqref{eq:OSMD-update}, $p^{t+1}$, can be found as
\begin{align*}
\check{p}^{t+1} & =  \arg\min_{ q \in \mathbb{R}^n_{+} } \, - \eta \langle q, \hat{u}^t \rangle + D_{\Phi} (q \, \Vert \, p^t), \\
p^{t+1} & =  \arg\min_{ q \in \mathcal{A} } \, D_{\Phi} (q \, \Vert \, \check{p}^{t+1})
\end{align*}
The optimality condition for $\check{p}^{t+1}$ implies that
\begin{equation}\label{eq:proof-prop1-1}
-\eta \hat{u}^t + \nabla \Phi \left( \check{p}^{t+1} \right) - \nabla \Phi \left( p^t \right) = 0
\end{equation}
By Lemma~\ref{lemma:convx-opt-cond}, the optimality condition for $p^{t+1}$ implies that
\begin{equation*}
\langle p - p^{t+1}, \nabla \Phi (p^{t+1}) - \nabla \Phi (\check{p}^{t+1}) \rangle \geq 0, \quad \text{for all } p \in \mathcal{A}.
\end{equation*}
Combining the last two displays, we have
\begin{equation}
\label{eq:proof-prop1-1-h2}
\langle p - p^{t+1}, -\eta \hat{u}^t + \nabla \Phi \left( p^{t+1} \right) - \nabla \Phi \left( p^t \right) \rangle \geq 0, \quad \text{for all } p \in \mathcal{A}.
\end{equation}
By Lemma~\ref{lemma:convx-opt-cond}, this is the optimality condition for $p^{t+1}$ to be the solution of~\eqref{eq:OSMD-update}.
Note that \eqref{eq:proof-prop1-1} implies that
\begin{equation*}
- \eta \hat{u}^t_i + \log \check{p}^{t+1}_i - \log p^t_i = 0 \quad \text{ for all } i \in [n].
\end{equation*}
Therefore, 
\[
\check{p}^{t+1}_i = p^t_i \exp \left( \eta \hat{u}^t_i \right) = \tilde{p}^{t+1}_i \quad \text{ for all } i \in [n].
\]
and the final result follows from Lemma~\ref{lemma:mirror-solver-lemma}.

\end{proof}

\subsection{Proof of Theorem~\ref{thm:regret-bound}}
\label{sec:proof-thm-1}

Our proof is based on \citet[Chapter 31.1]{lattimore2020bandit}.
Recall that we have
\begin{align*}
R_B & = \mathbb{E} \left[ \max_{ (a^t) \in \Gamma_m } K \sum^{T-1}_{t=0} \Delta U^t_{a^t} -  \sum^T_{t=0} \Delta U^t_\text{cm} \right] \\
&= \mathbb{E} \left[ \max_{ (a^t) \in \Gamma_m } \sum^{T-1}_{t=0} \sum^{K-1}_{k=0} \Delta U^t_{a^t} -  \sum^T_{t=0} \sum_{i \in S^t} \Delta U^t_i \right] \\
&= \mathbb{E} \left[ \max_{ (a^t) \in \Gamma_m } \sum^{T-1}_{t=0} \sum^{K-1}_{k=0} \left( \Delta U^t_{a^t} - \Delta U^t_{i^{t,k}} \right) \right].
\end{align*}
Let
\begin{equation*}
\left( a^0_{\star},\ldots,a^{T-1}_{\star} \right) = \argmax_{(a^t) \in \Gamma_m }  \sum^{T-1}_{t=0} \Delta U^t_{a^t} ,
\end{equation*}
then there exists $0=t_1<t_2<\ldots<t_m<t_{m+1}= T-1$ such that for any $j=1,\ldots,m$, $a^t_{\star}$ is a constant when $t_j \leq t <t_{j+1}$.
To simplify the notation, we denote $\bar{a}^j_{\star}=a^{t_j}_{\star}$, we then decompose the regret as
\begin{equation*}
R_B = \mathbb{E} \left[ \sum^m_{j=1} \sum^{t_{j+1}-1}_{t=t_j} \sum^{K-1}_{k=0} \left( \Delta U^t_{\bar{a}^j_{\star}}  - \Delta U^t_{i^{t,k}} \right) \right] = \sum^m_{j=1} \mathbb{E} \left[\mathbb{E} \left[ \sum^{t_{j+1}-1}_{t=t_j} \sum^{K-1}_{k=0} \left( \Delta U^t_{\bar{a}^j_{\star}}  - \Delta U^t_{i^{t,k}} \right)  \,\middle\vert \,\,p^{t_j} \right] \right].
\end{equation*}
Let $u^t=(\Delta U^t_1,\ldots,\Delta U^t_n)^{\top}$. Note that since $U(\cdot) \in [0,1]$, we have $\Delta U^t_i \in [-1,1]$ for all $i \in [n]$. Besides, we have
\begin{align*}
& \quad \mathbb{E} \left[ \sum^{t_{j+1}-1}_{t=t_j} \sum^{K-1}_{k=0} \left( \Delta U^t_{\bar{a}^j_{\star}}  - \Delta U^t_{i^{t,k}} \right) \,\middle\vert \,\,p^{t_j} \right] \\
& = \mathbb{E} \left[ \max_{p \in \mathcal{P}_{n-1}} \sum^{t_{j+1}-1}_{t=t_j} \sum^{K-1}_{k=0}  \left\langle p - p^t, u^t \right\rangle   \,\middle\vert \,\,p^{t_j} \right] \\
& \leq \mathbb{E} \left[ \max_{p \in \mathcal{A} } \sum^{t_{j+1}-1}_{t=t_j} \sum^{K-1}_{k=0}  \left\langle p - p^t, u^t \right\rangle \,\middle\vert \,\,p^{t_j} \right] + K \frac{\alpha}{n}(n-1) (t_{j+1}-t_j) \\
& = \mathbb{E} \left[ \max_{p \in \mathcal{A} } K \sum^{t_{j+1}-1}_{t=t_j} \left\langle p - p^t, \hat{u}^t \right\rangle \,\middle\vert \,\,p^{t_j} \right] + K \alpha\left( 1 - \frac{1}{n} \right) (t_{j+1}-t_j). \numberthis \label{eq:proof-helper-1}
\end{align*}
Since
\begin{equation}
\label{eq:prop-Bregman}
D_{\Phi} \left(x_1 \,\Vert\, x_2 \right) = D_{\Phi} \left(x_3 \,\Vert\, x_2 \right) + D_{\Phi} \left(x_1 \,\Vert\, x_3 \right) + \langle \nabla \Phi(x_2) - \nabla \Phi(x_3), x_3 - x_1 \rangle,
\qquad x_1,x_2,x_3\in\mathcal{D},   
\end{equation}
then by~\eqref{eq:proof-prop1-1-h2}, we have
\begin{align*}
\langle p - p^{t+1}, \hat{u}^t \rangle & \leq \frac{1}{\eta} \langle p - p^{t+1}, \nabla \Phi \left( p^{t+1} \right) - \nabla \Phi \left( p^t \right) \rangle \\
& = \frac{1}{\eta} \left( D_{\Phi} \left( p \,\Vert\, p^t \right) - D_{\Phi} \left( p \,\Vert\, p^{t+1} \right) - D_{\Phi} \left( p^{t+1} \,\Vert\, p^t \right)  \right).
\end{align*}
Thus, we have
\begingroup
\allowdisplaybreaks
\begin{align*}
& \quad \sum^{t_{j+1}-1}_{t=t_j}  \left\langle p - p^t, \hat{u}^t \right\rangle \\
& = \sum^{t_{j+1}-1}_{t=t_j}  \left\langle p^{t+1} -p^t, \hat{u}^t \right\rangle + \sum^{t_{j+1}-1}_{t=t_j} \left\langle p -p^{t+1}, \hat{u}^t \right\rangle \\
& \leq \sum^{t_{j+1}-1}_{t=t_j}  \left\langle p^{t+1} -p^t, \hat{u}^t \right\rangle  + \frac{1}{\eta} \sum^{t_{j+1}-1}_{t=t_j} D_{\Phi} \left( p \,\Vert\, p^t \right) - D_{\Phi} \left( p \,\Vert\, p^{t+1} \right) - D_{\Phi} \left( p^{t+1} \,\Vert\, p^t \right) \\
& = \sum^{t_{j+1}-1}_{t=t_j} \left\langle p^{t+1} -p^t, \hat{u}^t \right\rangle + \frac{1}{\eta} \left( D_{\Phi} \left( p \,\Vert\, p^{t_j} \right) - D_{\Phi} \left( p \,\Vert\, p^{t_{j+1}} \right) \right) - \frac{1}{\eta}  \sum^{t_{j+1}-1}_{t=t_j} D_{\Phi} \left( p^{t+1} \,\Vert\, p^t \right) \\
& \leq  \sum^{t_{j+1}-1}_{t=t_j}   \left\langle p^{t+1} -p^t, \hat{u}^t \right\rangle + \frac{1}{\eta} D_{\Phi} \left( p \,\Vert\, p^{t_j} \right)  - \frac{1}{\eta}  \sum^{t_{j+1}-1}_{t=t_j}  D_{\Phi} \left( p^{t+1} \,\Vert\, p^t \right) \numberthis \label{eq:proof-helper-2}
\end{align*}
\endgroup
By~\eqref{eq:proof-prop1-1}, we have
\begin{equation*}
\hat{u}^t = \frac{1}{\eta} \left(\nabla \Phi \left( \tilde{p}^{t+1} \right) - \nabla \Phi \left( p^t \right) \right),
\end{equation*}
and since~\eqref{eq:prop-Bregman},
we have
\begin{align*}
\left\langle p^{t+1} -p^t, \hat{u}^t \right\rangle & = \frac{1}{\eta} \left\langle p^{t+1} -p^t, \nabla \Phi \left( \tilde{p}^{t+1} \right) - \nabla \Phi \left( p^t \right)  \right\rangle \\
& = \frac{1}{\eta} \left( D_{\Phi} \left( p^t \,\Vert\, \tilde{p}^{t+1} \right)  + D_{\Phi} \left( p^{t+1} \,\Vert\, p^t \right) - D_{\Phi} \left( p^{t+1} \,\Vert\, \tilde{p}^{t+1} \right) \right) \\
& \leq \frac{1}{\eta} \left( D_{\Phi} \left( p^t \,\Vert\, \tilde{p}^{t+1} \right)  + D_{\Phi} \left( p^{t+1} \,\Vert\, p^t \right)  \right). \numberthis \label{eq:proof-helper-3}
\end{align*}
Combine \eqref{eq:proof-helper-2} and~\eqref{eq:proof-helper-3}, we have
\begin{equation}
\label{eq:proof-helper-4}
\sum^{t_{j+1}-1}_{t=t_j}  \left\langle p -p^t, \hat{u}^t \right\rangle \leq \frac{1}{\eta} \left( D_{\Phi} \left( p \,\Vert\, p^{t_j} \right) + \sum^{t_{j+1}-1}_{t=t_j} D_{\Phi} \left( p^t \,\Vert\, \tilde{p}^{t+1} \right) \right).
\end{equation}
Note that
\begin{equation*}
D_{\Phi} \left( p^t \,\Vert\, \tilde{p}^{t+1} \right) = \sum^n_{i=1} p^t_i \left( \exp \left( \eta \hat{u}^t_i \right) - 1 - \eta \hat{u}^t_i \right) \leq \frac{\eta^2}{2} \sum^n_{i=1} p^t_i (\hat{u}^t_i)^2 \leq \frac{\eta^2 n}{2},
\end{equation*}
thus we have
\begin{equation}
\label{eq:proof-helper-5}
\frac{1}{\eta} \sum^{t_{j+1}-1}_{t=t_j} D_{\Phi} \left( p^t \,\Vert\, \tilde{p}^{t+1} \right) \leq \frac{\eta n (t_{j+1}-t_j)}{2}.
\end{equation}
Combine \eqref{eq:proof-helper-1}, \eqref{eq:proof-helper-4}, and \eqref{eq:proof-helper-5}, we have
\begin{align*}
\mathbb{E} \left[ \sum^{t_{j+1}-1}_{t=t_j} \sum^{K-1}_{k=0} \left( \Delta U^{t,k}_{\bar{a}^j_{\star}}  - \Delta U^{t,k}_{i^{t,k}} \right) \,\middle\vert \,\,p^{t_j} \right] & \leq K \alpha\left( 1 - \frac{1}{n} \right) (t_{j+1}-t_j) + \frac{\eta n K (t_{j+1}-t_j)}{2} \\
& + K \mathbb{E} \left[ \max_{p \in \mathcal{A}} \frac{ D_{\Phi} \left( p \,\Vert\, p^{t_j} \right) }{\eta} \,\middle\vert \,\,p^{t_j}  \right].
\end{align*}
Note that $p^{t_j} \in \mathcal{A}$, thus we have $p^{t_j}_i \geq \alpha / n$ for all $i \in [n]$ and
\begin{equation*}
D_{\Phi} \left( p \,\Vert\, p^{t_j} \right) = \sum^n_{i=1} p_i \log p_i + \sum^n_{i=1} p_i \log \left( \frac{1}{p^{t_j}_i}  \right) \leq \log (n/\alpha),
\end{equation*}
which implies that
\begin{equation*}
\mathbb{E} \left[ \sum^{t_{j+1}-1}_{t=t_j} \sum^{K-1}_{k=0} \left( \Delta U^{t,k}_{\bar{a}^j_{\star}}  - \Delta U^{t,k}_{i^{t,k}} \right)  \,\middle\vert \,\,p^{t_j} \right] \leq K \alpha\left( 1 - \frac{1}{n} \right) \left( t_{j+1}-t_j \right) + \frac{\eta n K (t_{j+1}-t_j)}{2} + \frac{K \log(n/\alpha)}{\eta},
\end{equation*}
and
\begin{equation*}
R_B \leq \alpha K T  + \frac{\eta n K T}{2} + \frac{m K \log(n/\alpha)}{\eta}.
\end{equation*}
The final result then follows by $B=KT$.

\section{Useful lemmas}

\begin{lemma}\label{lemma:convx-opt-cond}
Suppose that $f$ is a differentiable convex function defined on $\text{dom}f$, and $\mathcal{X} \subseteq \text{dom}f$ is a closed convex set. Then $x$ is the minimizer of $f$ on $\mathcal{X}$ if and only if
\begin{equation*}
\nabla f (x)^{\top} (y-x) \geq 0 \quad \text{for all } y \in \mathcal{X}. 
\end{equation*}
\end{lemma}
\begin{proof}
See Section 4.2.3 of \citet{boyd2004convex}.
\end{proof}

\begin{lemma}[Based on Exercise 26.12 of \citet{lattimore2020bandit} and Lemma 4 of~\citet{zhao2021adaptive}]
\label{lemma:mirror-solver-lemma}
Let $\alpha \in [0,1]$, $\mathcal{A}=\mathcal{P}_{n-1} \cap [ \alpha/n, 1 ]^n$. For $y \in [0,\infty)^n$, let $x=\arg\min_{v \in \mathcal{A}} D_{\Phi}(v \Vert y)$. Suppose $y_1\leq y_2 \leq \dots \leq y_n$. Let $i^{\star}$ be the smallest value such that
\begin{equation*}
y_{i^{\star}} \left( 1 -  \frac{ i^{\star}-1 }{n} \alpha \right) > \frac{  \alpha}{n} \sum^n_{i=i^{\star}} y_i.
\end{equation*}
Then 
\begin{equation*}
x_i = 
\begin{dcases*}
\alpha / n & 
if $i < i^{\star}$ \\
\frac{(1 - \frac{i^{\star} -1}{n} \alpha ) y_i}{\sum^n_{j=i^{\star}} y_j } & otherwise.
\end{dcases*}
\end{equation*}
\end{lemma}
\begin{proof}
Consider the following constrained optimization problem:
\begin{align*}
\min_{ u \in [0,\infty)^n } \, & \sum^n_{i=1} u_i \log \frac{u_i}{y_i}, \\
\text{s.t. } & \sum^n_{i=1} u_i = 1, \\
& u_i \geq \frac{\alpha}{n}, \qquad i \in [n].
\end{align*}
Since $x$ is the solution to this problem, by the optimality condition, there exists $\lambda,\nu_1,\dots,\nu_n \in \mathbb{R}$ such that
\begin{align}
\log \frac{x_i}{y_i} + 1- \lambda - \nu_i & =  0, 
\qquad i \in [n], \label{eq:solv-lemma-1} \\
\sum^n_{i=1} x_i & = 1, \label{eq:solv-lemma-2} \\
x_i - \frac{\alpha}{n} & \geq 0, \qquad i \in [n], \label{eq:solv-lemma-3} \\
\nu_i & \geq 0,  \qquad i \in [n], \label{eq:solv-lemma-4} \\
\nu_i \left( x_i - \frac{\alpha}{n} \right) &= 0, 
\qquad i \in [n]. \label{eq:solv-lemma-5}
\end{align}
By~\eqref{eq:solv-lemma-1}, we have $x_i=y_i \exp ( -1 + \lambda + \nu_i )$. By~\eqref{eq:solv-lemma-4} and \eqref{eq:solv-lemma-5}, when $x_i=\alpha/n$, we have $x_i=y_i \exp ( -1 + \lambda + \nu_i ) \geq y_i \exp ( -1 + \lambda )$; when $x_i>\alpha/n$, we have $x_i = y_i \exp ( -1 + \lambda )$.
Assume that $x_1 = \dots = x_{i^{\star}-1} = \alpha/n < x_{i^{\star}} \leq \dots \leq x_n$. Then
\begin{equation*}
1 = \sum^n_{i=1} x_i = ( i^{\star} -1 ) \frac{\alpha}{n} + \exp (-1+\lambda) \cdot \sum^n_{i=i^{\star}} y_i,
\end{equation*}
which implies that
\begin{equation}\label{eq:lemma4-proof-helper1}
\exp (-1+\lambda) = \frac{ 1 - ( i^{\star} -1 ) \frac{\alpha}{n}  }{  \sum^n_{i=i^{\star}} y_i }.
\end{equation}
Thus, we have
\begin{equation*}
x_{i^{\star}} = y_{i^{\star}} \exp (-1+\lambda) = y_{i^{\star}} \frac{ 1 - ( i^{\star} -1 ) \frac{\alpha}{n}  }{  \sum^n_{i=i^{\star}} y_i } > \frac{\alpha}{n}, 
\end{equation*}
which implies that
\begin{equation}\label{eq:solv-lemma-6}
y_{i^{\star}} \left( 1 -  \frac{ i^{\star}-1 }{n} \alpha \right) > \frac{  \alpha}{n} \sum^n_{i=i^{\star}} y_i .
\end{equation}

To complete the proof, we then only need to show that
\begin{equation}\label{eq:solv-lemma-7}
y_{i^{\prime}} \left( 1 -  \frac{ i^{\prime}-1 }{n} \alpha \right) \leq \frac{  \alpha}{n} \sum^n_{i=i^{\prime}} y_i
\end{equation}
for all $1 \leq i^{\prime} \leq i^{\star}-1$. The result then follows from~\eqref{eq:solv-lemma-6} and \eqref{eq:solv-lemma-7}.
To prove\eqref{eq:solv-lemma-7}, recall that for any $1 \leq i^{\prime} \leq i^{\star}-1$, we have $\alpha/n=y_{i^{\prime}} \exp ( -1 + \lambda + \nu_{i^{\prime}} )$, and because $y_1 \leq \cdots \leq y_n$, we have $\nu_1 \geq \cdots \geq \nu_{i^{\star}-1}$. This way, we have
\begin{equation}\label{eq:lemma4-proof-helper2}
\left( i^{\star}-i^{\prime} \right) \frac{\alpha}{n} = \sum^{i^{\star}-1}_{i=i^{\prime}} y_i \exp \left( -1 + \lambda + \nu_i \right) \leq \exp \left( -1 + \lambda + \nu_{i^{\prime}} \right) \sum^{i^{\star}-1}_{i=i^{\prime}} y_i.
\end{equation}
On the other hand, by~\eqref{eq:lemma4-proof-helper1}, we have
\begin{equation}\label{eq:lemma4-proof-helper3}
1 - ( i^{\star} -1 ) \frac{\alpha}{n} = \exp (-1+\lambda) \sum^n_{i=i^{\star}} y_i \leq \exp \left( -1 + \lambda + \nu_{i^{\prime}} \right) \sum^n_{i=i^{\star}} y_i.
\end{equation}
Combining \eqref{eq:lemma4-proof-helper2} and \eqref{eq:lemma4-proof-helper3}, we have
\begin{align*}
\frac{ 1 - ( i^{\prime} -1 ) \frac{\alpha}{n}  }{  \sum^n_{i=i^{\prime}} y_i } &=
\frac{ 1 - ( i^{\star} -1 ) \frac{\alpha}{n} + ( i^{\star} -i^{\prime} ) \frac{\alpha}{n} }{  \sum^n_{i=i^{\prime}} y_i } \\
& \leq \frac{ \exp \left( -1 + \lambda + \nu_{i^{\prime}} \right) \left( \sum^{i^{\star}-1}_{i=i^{\prime}} y_i + \sum^n_{i=i^{\star}} y_i \right)  }{  \sum^n_{i=i^{\prime}} y_i } \\
& = \exp \left( -1 + \lambda + \nu_{i^{\prime}} \right) \\
& = \frac{ \frac{  \alpha}{n} }{ y_{i^{\prime}} }, 
\end{align*}
which then implies~\eqref{eq:solv-lemma-7}.
\end{proof}

\end{document}